\documentclass[twocolumn]{article}
\usepackage{ifthen}
\usepackage{tikz}
\usetikzlibrary{spy,backgrounds}
\usepackage{dsfont}
\usepackage{amsthm}
\usepackage[ruled,linesnumbered,noend]{algorithm2e}
\usepackage{csquotes}
\usepackage{glossaries}
\newtheorem{theorem}{Theorem}
\newtheorem{corollary}{Corollary}[theorem]
\input{glossary.gls}
\usepackage{amsmath}
\usepackage{amssymb}
\usepackage{mathtools}
\usepackage{siunitx}
\usepackage[left=1.6cm,right=1.6cm, top=2cm, bottom=2.5cm]{geometry}
\usepackage[hidelinks]{hyperref}
\usepackage[capitalize]{cleveref}
\usepackage{microtype}
\usepackage{xcolor}
\usepackage{booktabs}
\usepackage{multirow}
\usepackage{esint}
\definecolor{mplblue}{HTML}{1f77b4}
\definecolor{mplorange}{HTML}{ff7f0e}
\definecolor{mplgreen}{HTML}{2ca02c}
\definecolor{mplred}{HTML}{d62728}
\definecolor{mplpurple}{HTML}{9467bd}
\author{%
	Martin Zach\\\texttt{\{martin.zach,} \and%
	Thomas Pock\\\texttt{pock\}@icg.tugraz.at} \and%
	Erich Kobler\\\texttt{kobler@uni-bonn.de}\and%
	Antonin Chambolle\\\texttt{antonin.chambolle@ceremade.dauphine.fr}%
}
\title{Explicit Diffusion of Gaussian Mixture Model Based Image Priors}

\newcommand{\R}{\mathbb{R}}
\newcommand{\T}{\top}
\newcommand{\Id}{\mathrm{Id}}
\newcommand{\argm}{\,\cdot\,}
\newcommand{\grad}[1]{\nabla_{\mkern-4mu#1}}
\DeclareMathOperator{\trace}{Tr}

\DeclareMathOperator{\tr}{\mathrm{Tr}}
\DeclareMathOperator{\proj}{proj}
\DeclareMathOperator*{\argmin}{arg\,min}

\DeclareMathOperator{\diag}{diag}
\DeclarePairedDelimiter\norm{\lVert}{\rVert}

\begin{document}
\maketitle
\begin{abstract}
	In this work we tackle the problem of estimating the density \( f_X \) of a random variable \( X \) by successive smoothing, such that the smoothed random variable \( Y \) fulfills \( (\partial_t - \Delta_1)f_Y(\,\cdot\,, t) = 0 \), \( f_Y(\,\cdot\,, 0) = f_X \).
	With a focus on image processing, we propose a product/fields of experts model with Gaussian mixture experts that admits an analytic expression for $f_Y (\,\cdot\,, t)$ under an orthogonality constraint on the filters.
	This construction naturally allows the model to be trained simultaneously over the entire diffusion horizon using empirical Bayes.
	We show preliminary results on image denoising where our model leads to competitive results while being tractable, interpretable, and having only a small number of learnable parameters.
	As a byproduct, our model can be used for reliable noise estimation, allowing blind
	denoising of images corrupted by heteroscedastic noise.
\end{abstract}
\section{Introduction}
Consider the practical problem of estimating the probability density \( f_X : \mathcal{X} \to \R \) of a random variable \( X \) in \( \mathcal{X} \), given a set of data samples \( \{ x_i \}_{i=1}^N \) drawn from \( f_X \).\footnote{For notational convenience, throughout this article we do not make a distinction between the \emph{distribution} and \emph{density} of a random variable.}
This is a challenging problem in high dimension (e.g.\ for images of size \( M \times N \), i.e.\ \( \mathcal{X} = \R^{M \times N} \)), due to extremely sparsely populated regions.
A fruitful approach is to estimate the density at different times when undergoing a diffusion process.
Intuitively, the diffusion equilibrates high- and low-density regions over time, thus easing the estimation problem.

Let \( Y_t \) (carelessly) denote the random variable whose distribution is defined by diffusing \( f_X \) for some time \( t \).
We denote the density of \( Y_t \) by \( f_Y(\,\cdot\,, t) \), which fulfills the diffusion equation \( (\partial_t - \Delta_1)f_Y(\,\cdot\,, t) = 0 \), \( f_Y(\,\cdot\,, 0) = f_X \).
The empirical Bayes theory~\cite{robbins_empirical_1956} provides a machinery for reversing the diffusion process:
Given an instantiation of the random variable \( Y_t \), the Bayesian least-squares estimate of \( X \) can be expressed solely using \( f_Y(\,\cdot\,, t) \).
Importantly, this holds for all positive \( t \), as long as \( f_Y \) is is properly constructed.

In practice we wish to have a parametrized, trainable model of \( f_Y \), say \( f_\theta \) where \( \theta \) is a parameter vector, such that \( f_Y(x, t) \approx f_\theta(x, t) \) for all \( x \in \mathcal{X} \) and all \( t \in [0, \infty) \).
Recent choices for the family of functions \( f_\theta(\argm, t) \) were of practical nature:
Instead of an analytic expression for \( f_\theta \) at any time \( t \), authors proposed a time-conditioned network in the hope that it can learn to behave as if it had undergone a diffusion process.
Further, instead of worrying about the normalization \( \int_\mathcal{X} f_Y(\argm, t) = 1 \) for all \( t \in [0, \infty) \), usually they directly estimate the \emph{score} \( -\grad{1} \log f_Y(\argm, t) : \mathcal{X} \to \mathcal{X} \) with some network~\( s_\theta(\argm, t) : \mathcal{X} \to \mathcal{X} \).
This has the advantage that normalization constants vanish, but usually, the constraint \( \partial_j (s_\theta(\argm, t))_i = \partial_i (s_\theta(\argm, t))_j \) is not enforced in the architecture of \( s_\theta \).
Thus, \( s_\theta(\argm, t) \) is in general not the gradient of a scalar function (the negative-log-density it claims to model).

In this paper, we pursue a more principled approach.
Specifically, we leverage Gaussian mixture models to represent the popular product/field of experts model~\cite{hinton_training_2002,RoBl09} and show that under an orthogonality constraint of the associated filters, the diffusion of the model can be expressed analytically.

\section{Background}
In this section, we first emphasize the importance of the diffusion process in density estimation (and sampling) in high dimensions.
Then, we detail the relationship between diffusing the density function, empirical Bayes, and denoising score matching~\cite{vincent_connection_2011}.

\subsection{Diffusion Eases Density Estimation and Sampling}
Let \( f_X \) be a density on \( \mathcal{X} \subset \R^d\).
A major difficulty in estimating \( f_X \) with parametric models is that \( f_X \) is extremely sparsely populated in high dimensional spaces\footnote{Without any reference to samples \( x_i \sim f_X \), an equivalent statement may be that \( f_X \) is (close to) zero almost everywhere (in the layman --- not measure theoretic --- sense).}, i.e., \( d\gg1 \).
This phenomenon has many names, e.g.\ the curse of dimensionality or the manifold hypothesis~\cite{bengio_representation_2013}.
Thus, the learning problem is difficult, since meaningful gradients are rare.
Conversely, let us for the moment assume we have a model \( \tilde{f}_X \) that approximates \( f_X \) well.
In general, it is still very challenging to generate a set of points \( \{ x_i \}_{i=1}^I \) such that we can confidently say that the associated empirical density \( \frac{1}{I} \sum_{i=1}^I \delta(\argm - x_i) \) approximates \( \tilde{f}_X \).
This is because, in general, there does not exist a procedure to directly draw from \( \tilde{f}_X \), and (modern) \gls{mcmc} relies on the estimated gradients of \( \tilde{f}_X \) and, in practice, only works well for unimodal distributions~\cite{song_scorebased_2021}.

The isotropic diffusion process or heat equation
\begin{equation}
	(\partial_t - \Delta_1) f(\argm, t) = 0\ \text{with initial condition}\ f(\argm, 0) = f_X
	\label{eq:diff}
\end{equation}
equilibrates the density in \( f_X \), thus mitigating the challenges outlined above.
Here, \( \partial_t \) denotes \( \frac{\partial}{\partial_t} \) and \( \Delta_1 = \tr(\grad{1}^2)\) is the Laplace operator, where the \( 1 \) indicates application to the first argument.
We detail the evolution of \( f_X \) under this process and relations to empirical Bayes in~\cref{ssec:diffusion empirical bayes}.

\emph{Learning} \( f(\argm, t) \) for \( t \geq 0 \) is more stable since the diffusion \enquote{fills the space} with meaningful gradients~\cite{SoEr19}.
Of course, this assumes that for different times \( t_1 \) and \( t_2 \), the models of \( f(\argm, t_1) \) and \( f(\argm, t_2) \) are somehow related to each other.
As an example of this relation, the recently popularized noise-conditional score-network~\cite{song_scorebased_2021} shares convolution filters over time, but their input is transformed with a time-conditional instance normalization.
In this work, we make this relation explicit by considering a family of functions \( f(\argm, 0) \) for which \( f(\argm, t) \) can be expressed analytically.

For \emph{sampling}, \( f(\argm, t) \) for \( t > 0 \) can help by gradually moving samples towards high-density regions of \( f_X \), regardless of initialization.
To utilize this, a very simple idea with relations to simulated annealing~\cite{Kirkpatrick1983,song_scorebased_2021} is to have a pre-defined time schedule \( t_T > t_{T-1} > \ldots > t_1 > 0 \) and sample \( f(\argm, t_i) \), \( i = T, \dotsc, 0 \) (e.g.\ with Langevin dynamics~\cite{roberts_exponential_1996}) successively.
\subsection{Diffusion, Empirical Bayes, and Denoising Score Matching}%
\label{ssec:diffusion empirical bayes}
In this section, similar to the introduction, we again adopt the interpretation that the evolution in~\cref{eq:diff} defines the density of a smoothed random variable \( Y_t \).
That is, \( Y_t \) is a random variable with probability density \( f_Y(\argm, t) \), which fulfills \( (\partial_t - \Delta_1) f_Y(\argm, t) = 0 \) and \( f_Y(\argm, 0) = f_X \).
It is well known that the Green's function of~\cref{eq:diff} is a Gaussian (see e.g.~\cite{cole_green_2010}) with zero mean and variance \( \sigma^2(t) = 2t\Id\).
In other words, for \( t > 0 \) we can write \( f_Y(\argm, t) = G_{0,2t\Id} * f_X \), where 
\begin{equation}
G_{\mu,\Sigma}(x) = |2 \pi \Sigma|^{-1/2} \exp\left( -\frac{1}{2} (x - \mu)^\T\Sigma^{-1} (x - \mu) \right).
\label[type]{eq:gaussianPdf}
\end{equation}
Thus, the diffusion process constructs a (linear) \emph{scale space in the space of probability densities}.
In terms of the random variables, \( Y_t = X + \sqrt{2t}N \) where \( N \sim \mathcal{N}(0, \Id) \).
We next motivate how to estimate the corresponding instantiation of \( X \) which has \enquote{most likely} spawned an instantiation of \( Y_t \) using empirical Bayes.

In the school of empirical Bayes~\cite{robbins_empirical_1956}, we try to estimate a clean random variable given a corrupted instantiation, using only knowledge about the corrupted density.
In particular, for our setup, Miyasawa~\cite{miyasawa_empirical_1961} has shown that the Bayesian least-squares estimator \( x_{\text{EB}} \) for an instantiation \( y_t \) of \( Y_t \) is
\begin{equation}
	x_{\text{EB}}(y_t) = y_t + \sigma^2(t) \grad{1} \log f_Y(y_t, t),
	\label{eq:tweedie}
\end{equation}
which is also known as Tweedie's formula~\cite{efron_tweedie_2011}.
Raphan and Simoncelli~\cite{raphan_least_2011} extended the empirical Bayes framework to arbitrary corruptions and coined the term non-parametric empirical Bayes least-squares (NEBLS).

Recently, \cref{eq:tweedie} has been used frequently for parameter estimation in diffusion-based models.
Let \( \{ x_i \}_{i=1}^I \) be a dataset of \( I \)~samples drawn from \( f_X \) and \( Y_t \) governed by the considered diffusion process.
Thus, both the left- and right-hand side of~\cref{eq:tweedie} are \emph{known} --- in expectation.
This naturally leads to the loss function
\begin{equation}
	\min_\theta \int_{(0, \infty)} \mathbb{E}_{(x, y_t) \sim f_{X \times Y_t}} \norm{x - y_t - \sigma(t)^2 \grad{1} \log f_\theta(y_t, t)} \,\mathrm{d}t
	\label{eq:score}
\end{equation}
for estimating \( \theta \) such that \( f_\theta \approx f_Y \).
Here, \( f_{X\times Y_t} \) denotes the joint distribution of real and degraded points.
This learning problem is known as denoising score matching~\cite{hyvarinen_estimation_nodate,song_scorebased_2021,vincent_connection_2011}.
\section{Methods}\label{sec:methods}
In this section, we introduce a patch and convolutional model to approximate the prior distribution of natural images.
For both models, we present conditions such that they obey the diffusion process.

\subsection{Patch Model}
Patch-based prior models such as expected patch log likelihood (EPLL)~\cite{zoran_learning_2011} typically use \glspl{gmm} to approximately learn a prior for natural image patches.
Throughout this section, we approximate the density of image patches~\( p\in\R^a \) of size~\( a=b\times b \) as a product of \gls{gmm} experts, i.e.
\begin{equation}
	\tilde{f}_\theta(p, t) = Z(\{ k_j \}_{j=1}^J)^{-1}\prod_{j=1}^J \psi_j(\langle k_j, p \rangle, w_j, t),
	\label{eq:gmdm patch}
\end{equation}
in analogy to the product-of-experts model~\cite{hinton_training_2002}.
\( Z(\{ k_j \}_{j=1}^J) \) is required such that \( \tilde{f}_\theta \) is properly normalized.
Every expert~\( \psi_j : (\R \times \triangle^L \times [0, \infty)) \to \R^+ \) for \( j=1,\ldots,J \) models the density of associated filters~$k_j$ for all diffusion times~$t$ by an one-dimensional \gls{gmm} with \(L\)~components of the form
\begin{equation}
\psi_j(x,w_j,t) = \sum_{l=1}^{L} w_{jl} G_{\mu_l,\sigma_j^2(t)}(x).
\label{eq:expert}
\end{equation}
The weights of each expert \( w_j = (w_{j1}, \dotsc, w_{jL})^\T \) must satisfy the unit simplex constraint, i.e., \( w_j \in \triangle^L \), \( \triangle^L = \{ x \in \R^L : x_l \geq 0, \sum_{i=1}^{L} x_l = 1 \} \).
Although not necessary, we assume for simplicity that all~$\psi_j$ have the same number of components and each component has the same mean.
The discretization of \( \mu_l \) over the real line is fixed a priori in a uniform way and detailed in~\cref{ssec:implementation details}.
Further, the variances of all components of each expert are shared and are modeled as
\[
\sigma_j^2(t) = \sigma_0^2 + c_j 2t,
\]
where $\sigma_0$ is chosen to support the uniform discretization of the means~$\mu_l$ and $c_j\in\R_{++}$ are constants, to reflect the convolution effect of the diffusion process.


Next, we show how the diffusion process leads to the linear change of each expert's variance~$\sigma_j^2(t)$.
In detail, we exploit two well-known properties of \glspl{gmm}:
First, the product of \glspl{gmm} is again a \gls{gmm}, see e.g.~\cite{1591840}.
This allows us to work on highly expressive models that enable efficient \emph{evaluations} due to factorization.
Second, we use the fact that there exists an analytical solution to the diffusion equation if \( f_X \) is a \gls{gmm}:
The Green's function is a Gaussian with isotropic covariance \( 2t \mathrm{Id} \).
Hence, diffusion amounts to the convolution of two Gaussians for every component due to the linearity of convolution.
Using previous notation, if \( X \) is a random variable with normal distribution \( \mathcal{N}(\mu_X, \Sigma_X) \), then \( Y_t \) follows the distribution \( \mathcal{N}(\mu_X, \Sigma_X + 2t\Id) \).
In particular, the mean remains unchanged, thus it is sufficient to only adapt the variances in~\cref{eq:gmdm patch} linearly along with the diffusion time.

\begin{theorem}
	\( \tilde{f}(\argm, 0) \) is a homoscedastic \gls{gmm} on \( \R^a \) with \( L^J \) components, precision matrix
	\begin{equation}
		(\Sigma_a)^{-1} = \frac{1}{\sigma_0^2} \sum_{j=1}^J (k_j \otimes k_j).
		\label{eq:precision}
	\end{equation}
	and means \( \mu_{a,\hat{l}} = \Sigma_a \sum_{j=1}^J k_j\mu_{\hat{l}(j)} \).
	\label{th:gmm}
\end{theorem}
\begin{proof}
	By definition,
	\begin{equation}
		\begin{aligned}
			&\prod_{j=1}^J \psi(\langle k_j, p \rangle, w_{j}, 0) = \\&\prod_{j=1}^J \sum_{l=1}^{L} \frac{w_{jl}}{\sqrt{2\pi\sigma_0^2}} \exp\left( -\frac{1}{2\sigma_0^2}{(\langle k_j, p \rangle - \mu_l)}^2 \right).
		\end{aligned}
	\end{equation}
	Let \( \hat{l}(j) \) be a fixed but arbitrary selection from the index set \( \{ 1, \dotsc, L \} \) for each \( j \in \{ 1, \dotsc, J \} \).
	The general component of the above reads as
	\begin{equation}
		(2\pi\sigma_0^2)^{-\frac{J}{2}} \biggl( \prod_{j=1}^J w_{j\hat{l}(j)} \biggr) \exp\left( -\frac{1}{2\sigma_0^2} \sum_{j=1}^J (\langle k_j, p \rangle - \mu_{\hat{l}(j)})^2 \right).
	\end{equation}
	To find \( (\Sigma_a)^{-1} \), we complete the square as follows:
	Motivated by \( \grad{p} \norm{p - \mu_a}^2_{\Sigma_a^{-1}} / 2 = \Sigma_a^{-1} (p - \mu_a) \) we find \( \grad{p} \bigl( \frac{1}{2\sigma_0^2} \sum_{j=1}^J (\langle k_j, p \rangle - \mu_{\hat{l}(j)})^2 \bigr) = \frac{1}{\sigma_0^2} \sum_{j=1}^J (k_j \otimes k_j) p - k_j \mu_{\hat{l}(j)} \) and the theorem immediately follows.
\end{proof}

To get a tractable analytical expression for the diffusion process, we assume that the filters~\( k_j \) are pairwise orthogonal, i.e.\ for all \( i, j \in {\{ 1,\dotsc,J \}} \)
\begin{equation}
	\langle k_j, k_i \rangle = \begin{cases}
		0 & \text{if}\ i \neq j, \\
		\norm{k_j}^2 & \text{else}.
	\end{cases}\label{eq:ortho}
\end{equation}

\begin{theorem}[Patch diffusion]
	Under assumption~\cref{eq:ortho}, \( \tilde{f}(\,\cdot\,, t) \) satisfies the diffusion equation \( (\partial_t - \Delta_1) \tilde{f}(\argm, t) = 0 \) if \( \sigma_j^2(t) = \sigma_0^2 + \norm{k_j}^2 2t \).
	\label{th:diff local}
\end{theorem}
\begin{proof}
	Assuming~\cref{eq:ortho}, the Eigendecomposition of the precision matrix can be trivially constructed.
	In particular, \( (\Sigma_a)^{-1} = \sum_{j=1}^J \frac{\norm{k_j}^{2}}{\sigma_0^2} (\frac{k_j}{\norm{k_j}} \otimes \frac{k_j}{\norm{k_j}}) \), hence \( \Sigma_a = \sum_{j=1}^J \frac{\sigma_0^2}{\norm{k_j}^{2}} (\frac{k_j}{\norm{k_j}} \otimes \frac{k_j}{\norm{k_j}}) \).
	As discussed in~\cref{ssec:diffusion empirical bayes}, \( \Sigma_a \) evolves as \( \Sigma_a \mapsto \Sigma_a + 2t\mathrm{Id}_a \) under diffusion.
	Equivalently, for all \( j = 1, \ldots, J \) Eigenvalues, \( \frac{\sigma_0^2}{\norm{k_j}^{2}} \mapsto \frac{\sigma_0^2 + 2t\norm{k_j}^{2}}{\norm{k_j}^{2}} \).
	Recall that \( \sigma_0^2 \) is just \( \sigma_j^2(0) \).
	Thus, \( \tilde{f}(\argm,t) \) satisfies the diffusion equation if \( \sigma_j^2(t) = \sigma_0^2 + \norm{k_j}^2 2t \).
\end{proof}
\begin{corollary}
	With assumption~\eqref{eq:ortho} the potential functions \( \psi_j(\argm, w_j, t) \) in~\cref{eq:gmdm patch} model the marginal distribution of the random variable \( Z_{j,t} = \langle k_j, Y_t \rangle \).
	In addition,~\cref{eq:gmdm patch} is normalized when \( Z(\{ k_j \}_{j=1}^J)^{-1} = \prod_{j=1}^J \norm{k_j}^2 \).
	\label{cor:marginal}
\end{corollary}
\begin{proof}
	Consider one component of the resulting homoscedastic \gls{gmm}: \( \hat{Y}_t \sim \mathcal{N}(\mu_{a,\hat{l}}, \Sigma_a + 2t\mathrm{Id}_a) \).
	The distribution of \( \hat{Z}_{j, t} = \langle k_j, \hat{Y}_t \rangle \) is (see e.g.\ \cite{Gut2009} for a proof) \( \hat{Z}_{j, t} \sim \mathcal{N}(k_j^\top \mu_{a,\hat{l}}, k_j^\top (\Sigma_a + 2t\mathrm{Id}_a) k_j) = \mathcal{N}(\mu_{\hat{l}(j)}, \sigma_0^2  + 2t\norm{k_j}^2) \).
	The claim follows from the linear combination the different components.
\end{proof}

We note that~\cref{eq:precision} only specifies a covariance matrix if \( J = a \), otherwise the matrix is singular.
In the case \( J < a \), we restrict the analysis to the subspace \( \operatorname{span}(\{ k_1, \dotsc, k_J \})\).
In particular, we also assume that the diffusion process does not transport density out of this subspace.
\subsection{Convolutional Model}
To avoid the extraction and combination of patches in patch-based image priors and still account for the local nature of low-level image features, we describe a convolutional \gls{gmdm} next.
The following analysis assumes vectorized images \( x \in \R^n \) with \( n \)~pixels; the generalization to higher dimensions is straightforward.
In analogy to the patch-based model of the previous section, we extend the fields-of-experts model~\cite{RoBl09} to our considered diffusion setting by accounting for the diffusion time~\( t \) and obtain\footnote{For simplicity, we discard the normalization constant \( Z \), which is independent of \( t \).}
\begin{equation}
	f_\theta(x, t) = \prod_{i=1}^n \prod_{j=1}^J \psi_j((K_j x)_i, w_{j}, t).
	\label{eq:gmdm}
\end{equation}
Here, each expert~\( \psi_j \) models the density of convolution \emph{features} extracted by convolution kernels~\( {\{ k_j \}}_{j=1}^J \) of size \( a = b \times b \).
\( {\{ K_j \}}_{j=1}^J \subset \R^{n \times n}\) are the corresponding matrix representations and all convolutions are cyclic, i.e., \( K_j x \equiv k_j *_n x \), where \( *_n \) denotes a 2-dimensional convolution with cyclic boundary conditions.
Further, \( w_j \in \triangle^L \) are used the weight the components of each expert~\( \psi_j\) as in \cref{eq:expert}.
As in the patch model, it is sufficient to adapt the variances~$\sigma_j^2(t)$ by the diffusion time as the following analysis shows.

By definition for~$t=0$, we have
\begin{equation}
	f_\theta(x, 0) = \prod_{i=1}^n \prod_{j=1}^J \sum_{l=1}^{L} \frac{w_{jl}}{\sqrt{2\pi\sigma_0^2}} \exp\left(-\frac{((K_j x)_i - \mu_l)^2}{2\sigma_0^2}\right).
	\label{eq:conv gmm}
\end{equation}
First, we expand the product over the pixels
\begin{equation}
	\begin{aligned}
		f_\theta(x, 0) &= \prod_{j=1}^J \sum_{\hat l(i) = 1}^{L^n} (2\pi\sigma_0^2)^{-\frac{n}{2}} \\&\overline{w}_{j\hat l(i)} \exp\left(-\frac{\norm{(K_j x) - \mu_{\hat{l}(i)}}^2}{2\sigma_0^2}\right)
	\end{aligned}
\end{equation}
using the index map~\( \hat{l}(i) \) and \(\overline{w}_{j\hat{l}(i)} = \prod_{i=1}^I w_{j\hat{l}(i)} \).
Further, expanding over the features results in
\begin{equation}
	\begin{aligned}
		f_\theta(x, 0) &= \sum_{\hat\imath(i, j)=1}^{(L^n)^J}(2\pi\sigma_0^2)^{-\frac{nJ}{2}} \\&\overline{\overline{w}}_{\hat\imath(i,j)} \exp\left(-\frac{1}{2\sigma_0^2}\sum_{j=1}^J \norm{(K_j x) - \mu_{\hat{\imath}(i, j)}}^2\right),
		\label{eq:expanded}
	\end{aligned}
\end{equation}
where \( \overline{\overline{w}}_{\hat\imath(i,j)}=\prod_{j=1}^{J}\prod_{i=1}^{I} w_{\hat\imath(i,j)} \)
Observe that~\cref{eq:expanded} again describes a homoscedastic \gls{gmm} with precision \( \Sigma^{-1} = \frac{1}{\sigma_0^2} \sum_{j=1}^J K_j^\top K_j \) and means \( \tilde{\mu}_{\hat\imath(i, j)} = \Sigma \frac{1}{\sigma_0^2} \sum_{j=1}^J K_j^\top \mu_{\hat\imath(i, j)} \).
Due to the assumed boundary conditions, the Fourier transform diagonalizes the convolution matrices: \( K_j = F^* \diag(Fk_j) F \).
Thus, the precision matrix can be expressed as
\begin{equation}
	\Sigma^{-1} = F^*\diag\biggl(\sum_{j=1}^J \frac{|Fk_j|^2}{\sigma^2}\biggr) F
	\label{eq:fourier diagonalization}
\end{equation}
where we used \( F F^* = \mathrm{Id} \), \( \bar{z}z = |z|^2 \) and \( |\,\cdot\,| \) denotes the complex modulus acting element-wise on its argument.
We assume that the spectra of \( k_j \) have disjoint support, i.e.
\begin{equation}
	\Gamma_i \cap \Gamma_j = \emptyset\ \text{ if }\ i\neq j,
	\label{eq:disjoint}
\end{equation}
where \( \Gamma_j = \operatorname{supp} Fk_j \).
Note that, in analogy to the pair-wise orthogonality of the filters in the patch model~\cref{eq:ortho}, from this immediately follows that \( \langle Fk_j, Fk_i \rangle = 0 \) when \( i \neq j \).
In addition, we assume that the magnitude is constant over the support, i.e.
\begin{equation}
	|Fk_j| = \xi_j \mathds{1}_{\Gamma_j},
	\label{eq:constant}
\end{equation}
where \( \mathds{1}_A \) is the characteristic function of the set \( A \).
\begin{theorem}[Convolutional Diffusion]
	Under assumptions~\eqref{eq:disjoint} and~\eqref{eq:constant}, \( f(\argm, t) \) satisfies the diffusion equation \( (\partial_t - \Delta_1) f(\argm, t) = 0 \) if \( \bar{\sigma}_j^2(t) = \sigma_0^2 + \xi_j^2 2t \).
\end{theorem}
\begin{proof}
	The proof is in analogy to~\cref{th:diff local}.
	By~\cref{eq:fourier diagonalization}, under diffusion, \(  F^*\diag\bigl(\sum_{j=1}^J \frac{\sigma^2}{|Fk_j|^2} \bigr) F \mapsto F^*\diag\left(\frac{\sigma^2 + 2t\sum_{j=1}^J|Fk_j|^2 }{\sum_{j=1}^J |Fk_j|^2}\right) F \).
	Using~\cref{eq:disjoint} the inner sum decomposes as
	\begin{equation}
		\frac{\sigma_0^2 + 2t \sum_{j=1}^J |Fk_j|^2}{\sum_{j=1}^J |Fk_j|^2} = \sum_{j=1}^J \frac{\sigma_0^2 + 2t |Fk_j|^2}{|Fk_j|^2}
	\end{equation}
	and with~\cref{eq:constant} the numerator reduces to \( \sigma_0^2 + 2t\xi_j^2 \).
\end{proof}
\section{Numerical Results}
\subsection{Numerical Optimization}%
\label{ssec:implementation details}
For all experiments, \( \psi_j \) is a \( L = \num{125} \) component \gls{gmm}, with equidistant means \( \mu_l\) in the interval \( [-\gamma, \gamma] \), where we chose \( \gamma = 1 \).
To support the uniform discretization of the means, the shared standard deviation of the experts is \( \sigma_0 = \frac{2\gamma}{L - 1} \).
Assuming zero-mean filters of size \( b \times b \), we use \( J = b^2 - 1 \) filters.
Each component of the initial filters is independently drawn from a zero-mean Gaussian distribution with standard deviation \( b^{-1} \).
We avoid simplex projections by replacing \( w_j \) with learnable parameters \( \zeta_j \), from which \( w_j \) are computed using a soft-argmax \(w_{jl} = \frac{\exp{\zeta_{jl}}}{\sum_{l=1}^L \exp \zeta_{jl}},\) and initialize \( \zeta_{jl} = \frac{0.1\sqrt{\alpha}}{1 + \alpha\mu_l^2} \), where \( \alpha = \num{1000} \).

For the numerical experiments, \( f_X \) reflects the distribution of rotated and flipped \( b \times b \) patches from the \num{400} gray-scale images in the BSDS 500~\cite{martin_database_2001} training and test set, with each pixel in the interval \( [0, 1] \).
We optimize the parameters \( \theta = \{ (k_j, \zeta_j) \}_{j=1}^J \) in~\cref{eq:score} using the iPALM algorithm~\cite{pock_inertial_2016} with respect to a randomly chosen batch of size \( \num{3200} \) for \( \num{100000} \) steps.
We approximate the infinite-time diffusion process by uniformly drawing \( \sqrt{2t} \) from the interval \( [\num{0}, \num{0.4}] \).
We detail how we ensure the orthogonality of the filters during the iterations of iPALM in the next section.
\subsubsection{Enforcing Orthogonality}
Let \( K = [k_1, k_2, \dotsc, k_J] \in \R^{a \times J} \) denote the matrix obtained by horizontally stacking the filters.
We are interested in finding
\begin{equation}
	\operatorname{proj}_{\mathcal{O}}(K) = \argmin_{M \in \mathcal{O}} \norm{M - K}_F^2
\end{equation}
where \( \mathcal{O} = \{ X \in \R^{a \times J} : X^\top X = D^2 \} \), \( D = \diag(\lambda_1,\lambda_2,\dotsc,\lambda_J) \) is diagonal, and \( \norm{\,\cdot\,}_F \) is the Frobenius norm.
Since \( \operatorname{proj}_{\mathcal{O}}(K)^\top \operatorname{proj}_{\mathcal{O}}(K) = D^2 \) we can represent it as \( \operatorname{proj}_{\mathcal{O}}(K) = OD \) with \( O \) semi-unitary (\( O^\top O = \mathrm{Id} \)).
Other than positivity, we do not place any restrictions on \( \lambda_1, \dotsc, \lambda_J \), as these are related to the precision in our model.
Thus, we rewrite the objective
\begin{equation}
	\begin{aligned}
		&\operatorname{proj}_{\mathcal{O}}(K) = \argmin_{\substack{O^\top O = \mathrm{Id}_J \\ D = \diag(\lambda_1,\dotsc,\lambda_J)}} \\ &\big\{ \norm{OD - K}_F^2  = \norm{K}_F^2 - 2 \langle K, OD \rangle_F + \norm{D}_F^2 \big\}
	\end{aligned}
\end{equation}
where \( \langle \,\cdot\,, \,\cdot\, \rangle_F \) is the Frobenius inner product.

We propose the following alternating minimization scheme for finding \( O \) and \( D \).
The solution for the reduced sub-problem in \( O \) can be computed by setting \( O = U \), using the polar decomposition of \( DK^\top = UP \), where \( U \in \R^{J \times a}\) is semi-unitary (\( U^\top U = \mathrm{Id}_{a} \)) and \( P = P^\top \in \mathbb{S}^a_+ \).
The sub-problem in \( D \) is solved by setting \( D_{i,i} = \bigl((O^\top K)_{i,i}\bigr)_{+} \).
The algorithm is summarized in~\cref{alg:orthogonalizing}, where we have empirically observed fast convergence; \( B = 3 \) steps already yielded satisfactory results.
A theoretical analysis of the algorithm is presented in the supplemental material.
\begin{algorithm}
	\DontPrintSemicolon
	\SetKwInOut{Output}{Output}
	\SetKwInOut{Input}{Input}
	\Input{\( K = [k_1, \dotsc, k_J] \in \R^{a \times J} \), \( B \in \mathbb{N} \), \( D^{(1)} = \mathrm{Id}_J \)}
	\Output{\( O^{(B)}D^{(B)} = \proj_{\mathcal{O}}(K) \)}
	\For{\( b \in 1, \dotsc, B - 1 \)}{
		\( U^{(b)}P^{(b)} = D^{(b)}K^\top \)\tcp*{Polar decomposition}
		\( O^{(b+1)} = U^{(b)} \)\;
		\( D^{(b+1)}_{i,i} = \bigl(((O^{(b+1)})^\top K)_{i, i} \bigr)_+ \)\;
	}
	\caption{%
		Algorithm for orthogonalizing a set of filters \( K \).
	}%
	\label{alg:orthogonalizing}
\end{algorithm}

To visually evaluate whether our learned model matches the empirical marginal densities for any diffusion time \( t \), we plot them in~\cref{fig:patch results}.
At the top, the learned \( 7 \times 7 \) orthogonal filters \( k_j \) are depicted, the associated learned potential functions \( -\log \psi_j \) are shown below.
Indeed, they match the empirical marginal responses
\begin{equation}
	h_j(z, t) = -\log \mathbb{E}_{p \sim f_X} \delta(z - \langle k_j, p \rangle)
\end{equation}
visualized at the bottom almost perfectly even at the low-density tails.
In accordance to~\cref{th:diff local}, the potentials barely change with \( t \) when \( \norm{k_j} \) is small.
Conversely, when \( \norm{k_j} \) is large, the change is much more drastic.
We observe the same for \( 15 \times 15 \) filters, as shown in the supplementary material.
\begin{figure*}
	\newlength{\hheight}
	\settoheight{\hheight}{A}
	\centering
	\includegraphics[trim=4.5cm .7cm 4cm .6cm, clip, width=\textwidth]{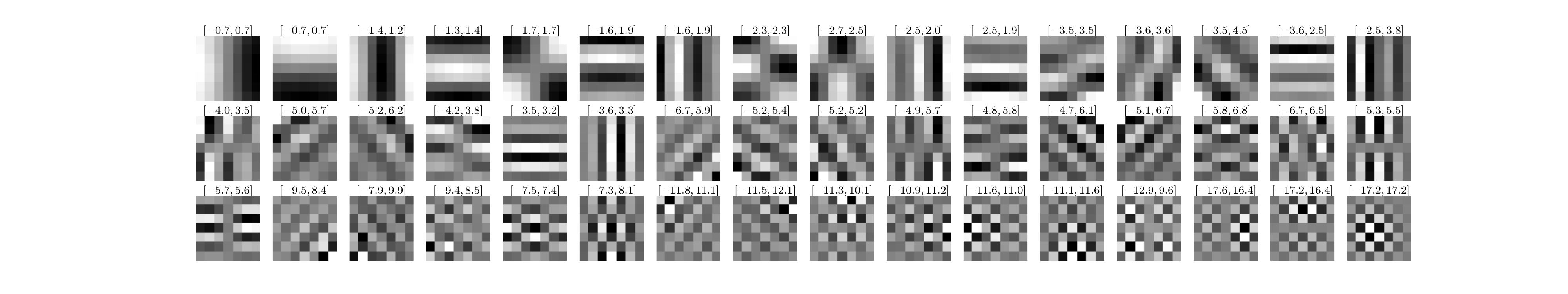}
	\includegraphics[trim=4.5cm .3cm 4cm .7cm, clip, width=\textwidth]{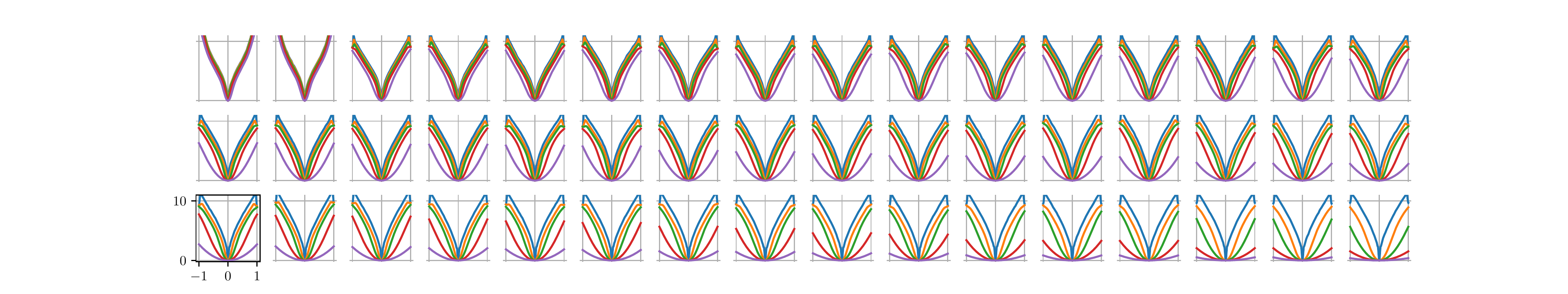}
	\rule[3mm]{\textwidth}{.3mm}\vspace*{-3mm}
	\vspace*{-3mm}
	\includegraphics[trim=4.5cm .3cm 4cm .7cm, clip, width=\textwidth]{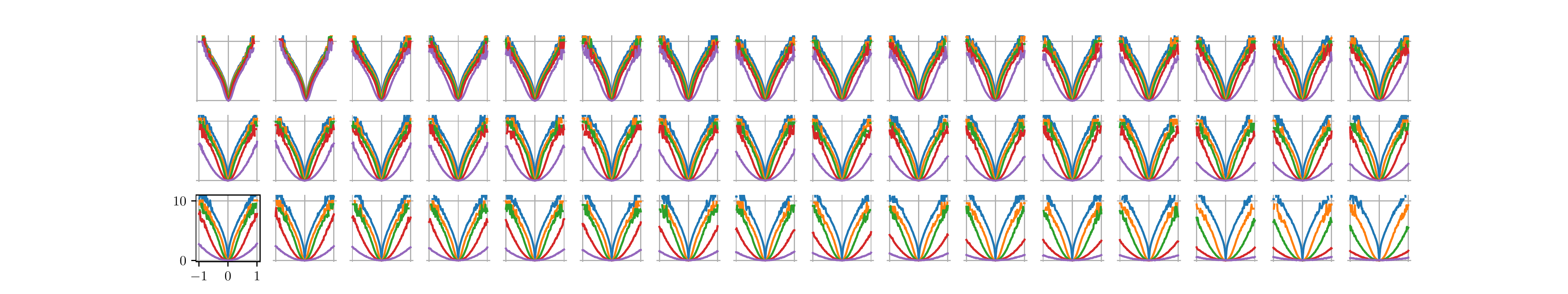}
	\vspace*{-3mm}
	\hspace*{8.2cm}{\resizebox{4cm}{!}{\(\sigma(t)=\)
		\foreach \ccolor\ssigma in {mplblue/0, mplorange/0.025, mplgreen/0.05, mplred/0.1, mplpurple/0.2}
		{
			\tikz[baseline=\hheight]{\draw[\ccolor, ultra thick](0,0.3) -- (.5,.3);}\num{\ssigma}
		}
	}}
	\caption{%
		Learned filters \( k_j \) (top, the intervals show the values of black and white respectively, amplified by a factor of \num{10}) and potential functions \( -\log \psi_j \) (middle).
		On the bottom the empirical marginal filter response histograms are drawn.
	}%
	\label{fig:patch results}
\end{figure*}
\subsection{Sampling}
A direct consequence of~\cref{cor:marginal} is that our model admits a simple sampling procedure:
The statistical independence of the components allows to draw random patches by \(Y_t = \sum_{j=1}^J \frac{k_j}{\norm{k_j}^2} Z_{j, t},\) where \( Z_{j, t} \) is sampled from the one-dimensional \gls{gmm} \( \psi_j \).
The samples in~\cref{fig:patch generation results} indicate a good match over a wide range of \( t \).
However, for small \( t \) the generated patches appear slightly noisy, which is due to a over-smooth approximation of the sharply peaked marginals around \( 0 \).
\begin{figure*}
	\centering
	\def\wwidth{3.6cm}
	\begin{tikzpicture}
		\foreach [count=\isigma] \ssigma/\llabel/\ccolor in {0.000/0/mplblue, 0.025/0.025/mplorange, 0.050/0.05/mplgreen, 0.100/0.1/mplred, 0.200/0.2/mplpurple}{
			\foreach [count=\iwhich] \which in {true, analytical}{
				\node at (\isigma*\wwidth+\isigma*1, -\iwhich*\wwidth/2.3-\iwhich*5) {\includegraphics[width=\wwidth]{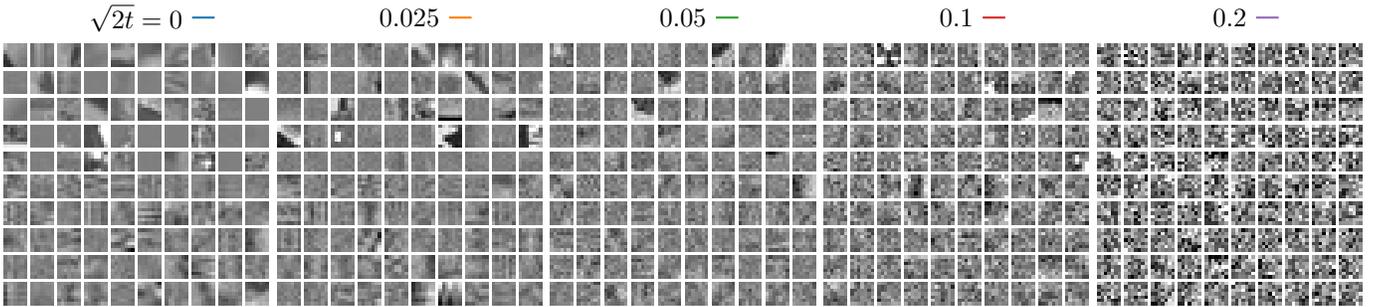}};
			}
			\node (anno\isigma) at (\isigma*\wwidth+\isigma, -.52){\ifthenelse{\isigma=1}{\( \sqrt{2t} = \llabel \)}{\( \llabel \)}};
			\draw [\ccolor, thick] (anno\isigma.east) -- ++(0.3, 0);
		}
	\end{tikzpicture}
	\caption{%
		Samples from the random variable \( Y_t \) (top) and generated patches (bottom).
	}%
	\label{fig:patch generation results}
\end{figure*}
\subsection{Image Denoising}
For the denoising experiments, we use the \num{68} test images from~\cite{martin_database_2001}.
To exploit our prior for denoising, we employ empirical Bayes-patch averaging (EB-PA) and the half-quadratic splitting (HQS) algorithm~\cite{zoran_learning_2011}.
In HQS, we approximate the solution to the inner MAP problem with one empirical Bayes step on \( \tilde{f}_\theta(\argm, t) \), and set \( \beta = \frac{1}{2t} \) for using a predefined schedule for \( t \).
The quantitative analysis in~\cref{tab:denosing} shows competitive performance, especially given the relatively small number of parameters in our model.
\begin{table*}
	\centering
	\caption{%
		Quantitative denoising results.
		Reference numbers are taken from~\cite{zoran_learning_2011}.
	}%
	\label{tab:denosing}
	\begin{tabular}{S[table-format=3]*{2}{S[table-format=2.2,round-mode=places,round-precision=2]}S[table-column-width=1.1cm,table-format=2.2,round-mode=places,round-precision=2]|*{2}{S[table-column-width=1.1cm,table-format=2.2,round-mode=places,round-precision=2]}|S[table-column-width=1.1cm,table-format=2.2,round-mode=places,round-precision=2]}
		\toprule
		{\multirow{2}{*}{\( 255\sigma \)}} & {\multirow{2}{*}{FoE~\cite{RoBl09}}} & {\multirow{2}{*}{GMM-EPLL~\cite{zoran_learning_2011}}} & \multicolumn{4}{c}{GMDM (\( b = 7 \mid b = 15\))} \\\cmidrule(l{1.5em}r{1.5em}){4-7}
						   & & & \multicolumn{2}{c}{EB-PA} & \multicolumn{2}{c}{HQS}  \\\midrule
		15 & 30.18 & 31.21 & 30.002985 & 30.33518791 & 30.37267303 & 30.68688393 \\
		25 & 27.77 & 28.72 & 27.47083664 & 27.80785942 & 28.12999535 & 28.38670158\\
		50 & 23.29 & 25.72 & 24.61125565 & 24.96177673 & 25.3236866 & 25.50110435 \\
		100 & 16.68 & 23.19 & 22.14141464 & 22.73509979 & 23.06726837 & 23.12487411\\\midrule
		{\#Params} & {648} & {819200} & {8352} & {78400} & {8352} & {78400} \\\bottomrule
	\end{tabular}
\end{table*}
\begin{figure}
	\centering
	\resizebox{\columnwidth}{!}{%
	\begin{tikzpicture}
		\def\wwidth{3cm}
		\foreach [count=\ii] \wwhich/\yoff in {noisy/0, estimate/0, denoised/-2.1, diff/-2.1}{
			\pgfmathsetmacro{\xoff}{mod(\ii-1,2)*3.6+mod(\ii-1,2)*.5}
			\begin{scope}[spy using outlines={rectangle, magnification=3, width=1.cm, height=2cm, connect spies}]
				\node at (\xoff, \yoff) {\includegraphics[width=\wwidth]{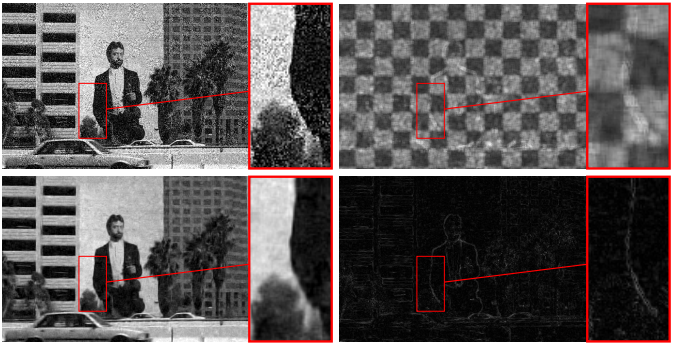}};
				\spy [red] on (\xoff-.4, \yoff-.3) in node [left] at (\xoff+2.5, \yoff);
			\end{scope}
		}
\end{tikzpicture}}
	\caption{%
		Blind denoising: Image corrupted by heteroscedastic Gaussian noise in a checkerboard-pattern (standard deviation \num{0.1} and \num{0.2}), noise estimate, EB-PA denoising result, and the difference to the reference image.
	}%
	\label{fig:blind denosing}
\end{figure}
\subsection{Noise Estimation and Blind Image Denoising}
The construction of our model allows us to interpret \( \tilde{f}_\theta(\argm, t) \) as a time-conditional likelihood density.
\cref{fig:blind denosing} shows that we can utilize our model for heteroscedastic blind denoising, by estimating pixel-wise noise and using EB-PA.
\cref{fig:noise estimation} shows the expected negative-log density over a range of \( \sigma \) and \( \sqrt{2t} \).
For visualization purposes, we normalized the negative-log density to have a minimum of zero over \( t \): \( l_\theta(x, t) = -\log\tilde{f}_\theta(x, t) - (\max_t \log \tilde{f}_\theta(x, t)) \).
The estimate \( \sigma \mapsto \argmin_t \mathbb{E}_{p \sim f_X, \eta \sim \mathcal{N}(0, \mathrm{Id})} l_\theta(p + \sigma\eta, t) \) is a very good match to \( \sigma \mapsto \sqrt{2t} \).
\begin{figure}
	\centering
	\begin{tikzpicture}
		\node at (0, 0) {\includegraphics[trim=.9cm 2cm 1cm 2cm, clip, width=.95\columnwidth]{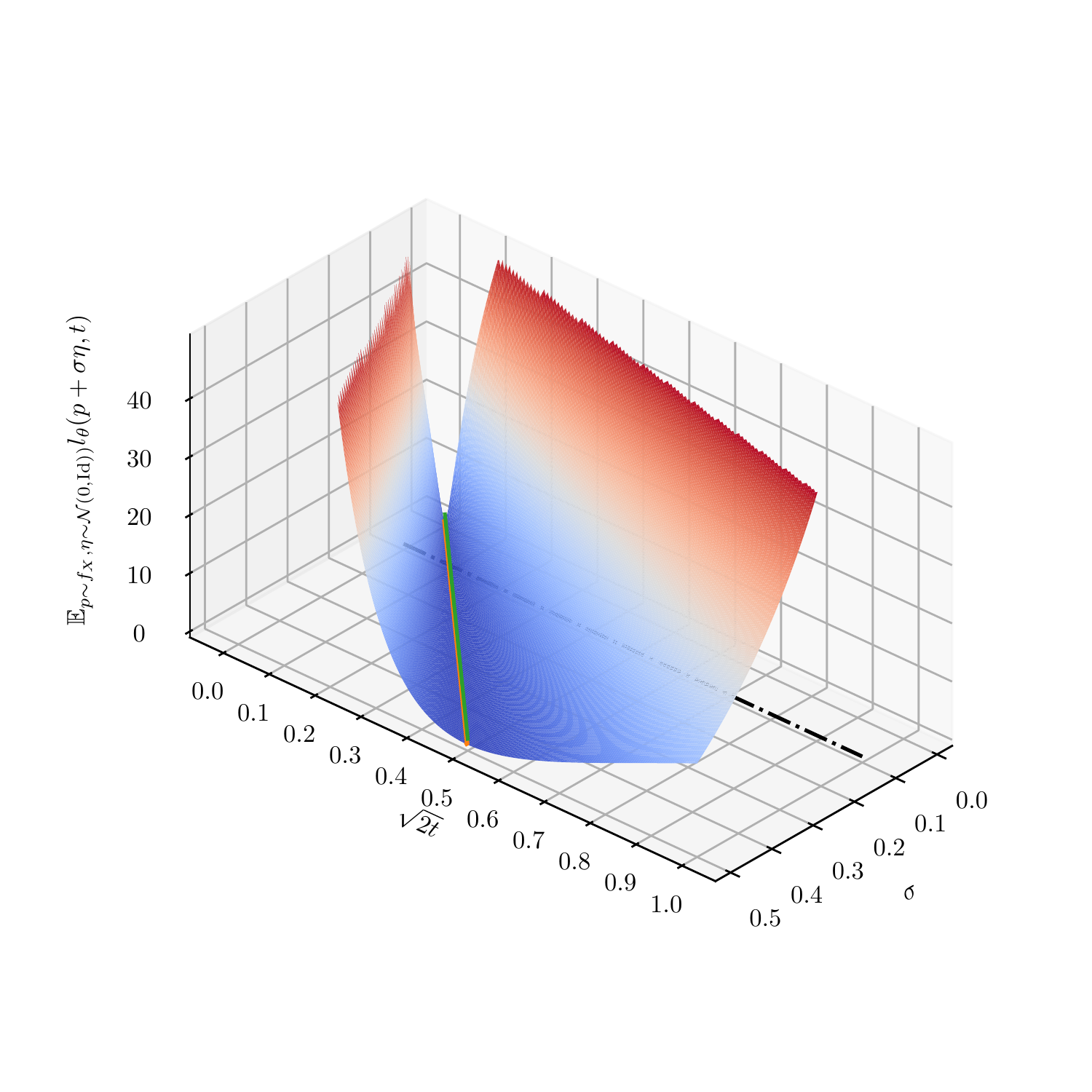}};
		\node at (0, -6.5) {\includegraphics[trim=.5cm .1cm 1cm 1.2cm, clip, width=.95\columnwidth]{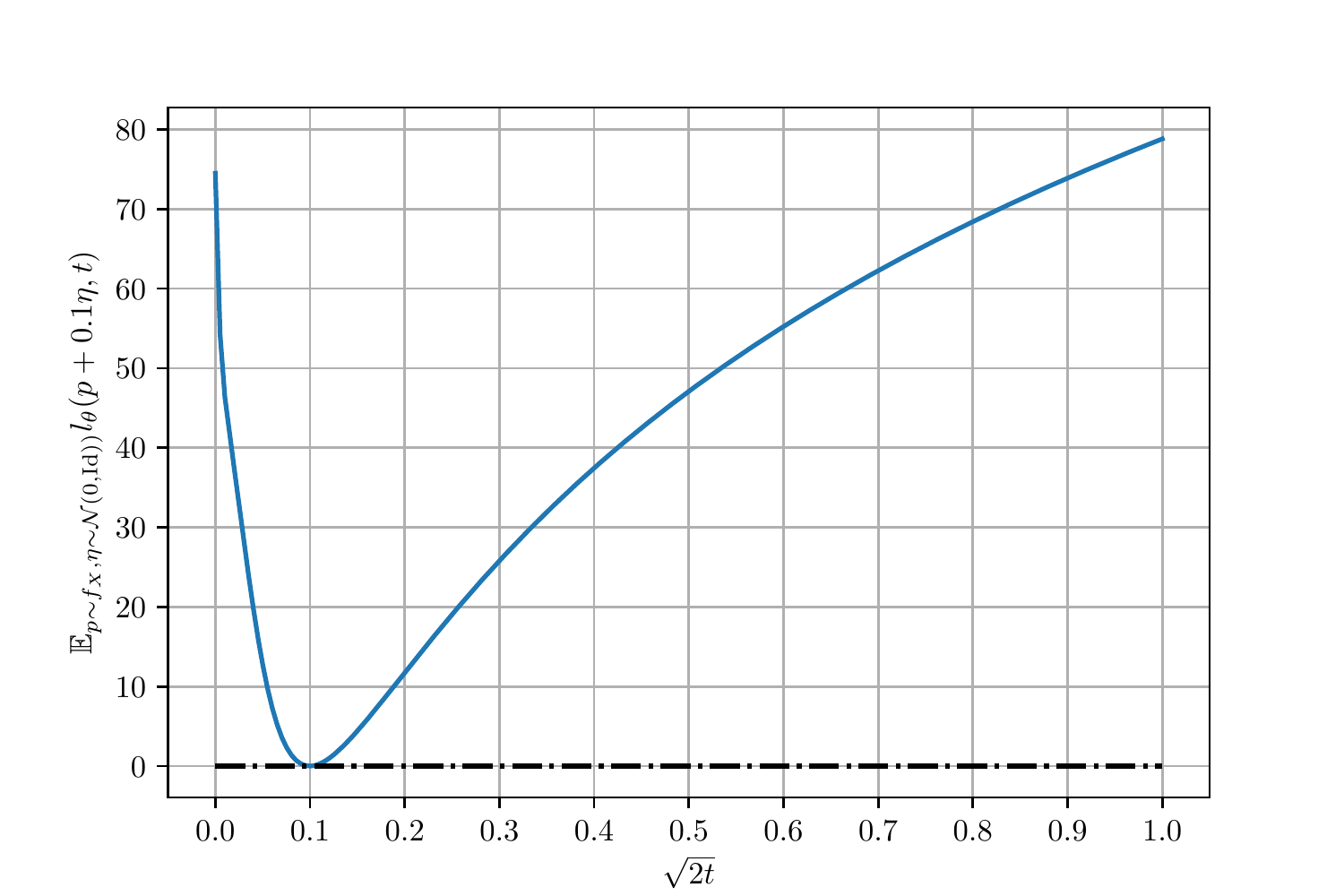}};
	\end{tikzpicture}
	\caption{%
		Top:
		Expected (zero-min-normalized) negative-log density along with the noise estimate \protect\tikz[baseline=\hheight]\protect\draw[mplorange, thick] (0, 0.1) -- ++(0.4, 0); \( \sigma \mapsto \argmin_t \mathbb{E}_{p \sim f_X, \eta \sim \mathcal{N}(0, \mathrm{Id})}l_\theta(p+\sigma\eta, t) \) and \protect\tikz[baseline=\hheight]\protect\draw [mplgreen, thick] (0, 0.1) -- ++(0.4, 0); \( \sigma \mapsto \sqrt{2t} \).
		Bottom: The slice at \( \sigma = 0.1 \).
	}%
	\label{fig:noise estimation}
\end{figure}

\section{Conclusion}
In this paper, we introduced \glspl{gmdm} as products of \glspl{gmm} on filter responses that allow for an explicit solution of the diffusion equation of the associated density.
Our explicit formulation enables learning of product/field-of-experts-like image priors simultaneously for all diffusion times using denoising score matching.
Our numerical results demonstrated that \glspl{gmdm} capture the statistics of natural image patches well for any noise level and hence are suitable for heteroscedastic (blind) image denoising.
In future work, we plan to extend the numerical evaluation to the convolutional model and apply our framework to challenging inverse problems in medical imaging.
\appendix
\section{Preliminary Theoretical Analysis of Projection Algorithm}
The problem is given $K\in\R^{m\times n}$, $n\le m$, find $A\in \R^{m\times n}$ such
that $A^T A$ is a diagonal $n\times n$ matrix which minimizes $\|A-K\|_2$
We decompose $K=U^K\Sigma^K (V^K)^T$ and $A=U\Sigma V^T$, with
$U,V$ in $SO(m)$, $SO(n)$ respectively and $\Sigma$ $m\times n$ matrices
with $\Sigma_{i,j}=0$ if $i\neq j$, and
$\Sigma_{i,j}=\sigma_i\ge 0$ for $i=1,\dots,n$
(hence if $m=n$ $\Sigma$ is diagonal,
if $m>n$ it is made of a $n\times n$ diagonal matrix ``above'' a $(m-n)\times n$
null matrix).

The constraints reads $A^TA = V \Sigma^2  V^T$ is diagonal, so that $A$ should
have a SVD representation with $V=I$, we assume now it has the form $U\Sigma$.

Then, $\|A-K\| = \|U\Sigma - U^K\Sigma^K (V^K)^T\| =
\|(U^K)^TU\Sigma - \Sigma^K (V^K)^T\|$ so without loss of generality, we can
assume that $K$ has the form $\Sigma^K (V^K)^T$, and we consider the problem:
\begin{equation}\label{eq:equiv}
  \min_{U,\Sigma} \|U\Sigma - \tilde K\|
\end{equation}
where $\tilde K = \Sigma^K (V^K)^T$.
Observe that
\[
  \|U\Sigma - \tilde K\|^2 = \|\Sigma\|^2 - 2(U\Sigma)\cdot \tilde K + \|\Sigma^K\|^2
\]
The above objective is easily minimized wr $\Sigma$ or $U$.
For $\Sigma$, since $(U\Sigma)\cdot\tilde K = \Sigma\cdot (U^T\tilde K)$ the solution is
\[
  \sigma_i = (U^T\tilde K)_{i,i} = \sum_{k=1}^n u_{k,i} \sigma^K_k v^K_{i,k}
\]
and the objective becomes (the first term below is constant and could be removed):
\[
  \min_{U} \|\Sigma^K\|^2 - \sum_{i=1}^n (U^T\tilde K)_{i,i}^2.
\]

\paragraph{Remark: Frank-Wolfe type method} Letting $f(U):=
\|\Sigma^K\|^2 - \sum_{i=1}^n (U^T\tilde K)_{i,i}^2$, one
has $\nabla f(U)\cdot M = -2\sum_{j=1}^n (U^T\tilde K)_{j,j} (\sum_{i=1}^m M_{i,j}\tilde K_{i,j})$, that is
\[
  (\nabla f(U))_{i,j} = -2 (U^T\tilde K)_{j,j} \tilde K_{i,j} = -2\sigma_j\tilde K_{i,j}.
\]
for $j\le n$, and $0$ for $j> n$.
This means that the alternating minimization algorithm can be viewed as
a Frank-Wolfe type method on $f$: starting from $U^0=I$, one
finds $U^{n+1}$ by minimizing $\nabla f(U^n)\cdot U$ for $U\in SO(m)$.
A stationary point is clearly a local minimum, yet as $f$ is concave,
it is not clear that it is a global minimum.

On the other hand, minimizing wr $U$ means solving
\[
  \max_{U^T U=I} U  \cdot (\tilde K\Sigma^T) 
\]
since $ (U\Sigma)\cdot \tilde K = \trace U\Sigma \tilde K^T = \trace U (\tilde K\Sigma^T)^T=
U\cdot (\tilde K\Sigma^T)$.

Now, given $M$ a $m\times m$ matrix, $\max_{O^T O} O\cdot M= \|M\|_1$ is the sum of the
singular values. Indeed if $M=U\Sigma V^T$ with $U,V\in SO(m)$, $O\cdot M = O\cdot (U\Sigma V^T)
= (U^TO V)\odot \Sigma$ and the max is reached for $O= UV^T$, with value $\sum_i \sigma_i$.

Hence the problem above is solved for $U = U^\Sigma (V^\Sigma)^T$ where $U^\Sigma,V^\Sigma$
are the $m\times m$ orthogonal matrices arising in the SVD representation of $\tilde K\Sigma^T$.
Then, the objective becomes:
\[
  \min_{\Sigma}
  \|\Sigma\|^2 - 2\|\tilde K\Sigma^T\|_1 + \|\Sigma^K\|^2
\]

Now we show that in fact we can work only with smaller, $n\times n$ matrices,
using the  particular form of $\tilde K$. Indeed,
$(\tilde K\Sigma^T)_{i,j} = \sum_{l=1}^n\sum_{k=1}^n\sigma^K_i \delta_{i,l}(V^K)_{k,l} \sigma_k \delta_{j,k}$
is $0$ if $i>n$ or $j>n$, and $\sigma^K_i (v^K)_{j,i}\sigma_j$ else.
Denoting $(\tilde K\Sigma^T)_n$ this reduced $n\times n$ matrix it is obvious that
it has the same singular values as $\tilde K\Sigma^T$ and hence, the objective
can be reduced to
\[
  \min_{(\sigma_i)_i} \sum_{i=1}^n \sigma_i^2  - 2 \|(\sigma_i \sigma^K_j v^K_{i,j})_{i,j=1}^n\|_1
  + \|\Sigma^K\|^2
\]
equivalently, we can replace the problem~\eqref{eq:equiv} with the same problem,
yet with smaller $n\times n$ matrices, and in particular $U\in SO(n)$ instead
of $SO(m)$.

Assume now we iterate by alternatively minimizing over $U$ and $\Sigma$ and
suppose we reach a fixed point. Then, consider
$f(U)=\|\Sigma^K\|^2-\sum_{i=1}^n (U^T\tilde K_n)^2_{i,i}$ where $\tilde K_n = (\sigma^K_i v^K_{i,j})_{i=1}^n$.
One has for small $t$:
\begin{equation}
	\begin{aligned}
	  &f(U+tM) = -\sum_i ((U^T+tM^T)\tilde K_n)^2_{i,i} = -\sum_i (U^T\tilde K_n)^2_{i,i} \\&-2t\sum_i(U^T\tilde K_n)_{i,i} \Big(\sum_j M_{j,i}(\tilde K_n)_{j,i}\Big) + o(t)
	\end{aligned}
\end{equation}
showing that $(\nabla f(U))_{i,j} = -2(U^T\tilde K_n)_{j,j} (\tilde K_n)_{i,j}= -2\sigma_j(\tilde K_n)_{i,j}=-2\tilde K_n\Sigma$ (where now $\Sigma=\Sigma^T$ since it is a $n\times n$ diagonal
matrix).

Now if $M$ is tangent to $SO(n)$, one has $(U+tM)^T(U+tM) = I + t (M^TU+U^TM) + t^2M^TM= I+o(t)$
so that $M^TU+U^TM=0$. Being $(U,\Sigma)$ a fixed point one has
$U= U^\Sigma (V^\Sigma)^T$ where $\tilde K_n\Sigma = U^\Sigma D(V^\Sigma)^T$ for some diagonal
matrix $D$. Then,
\begin{multline*}
  \nabla f(U)\cdot M = -2\trace(M^T U^\Sigma D(V^\Sigma)^T)
  \\= -2 \trace( M^TU V^\Sigma D(V^\Sigma)^T) = 2\trace (U^TM V^\Sigma D(V^\Sigma)^T)
  \\ = 2\trace (M V^\Sigma D (U^\Sigma)^T = 2\trace(M (\tilde K_n\Sigma)^T)= -\nabla f(U)\cdot M
\end{multline*}
so that $\nabla f(U)\cdot M=0$ and $U$ is a critical point of $f$ on $SO(n)$.
\section{Additional Experiments}
\begin{figure*}
	\includegraphics[trim=4.5cm .7cm 4cm .6cm, clip, width=\textwidth]{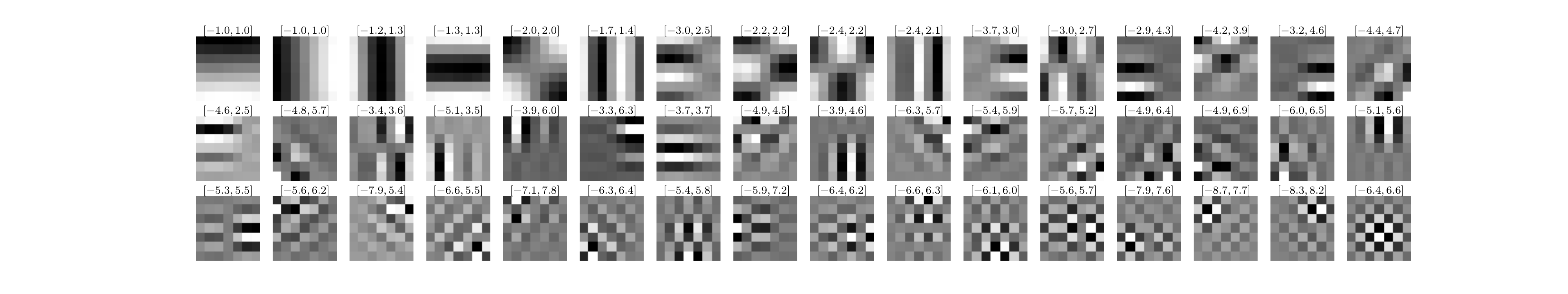}
	\includegraphics[trim=4.5cm .3cm 4cm .7cm, clip, width=\textwidth]{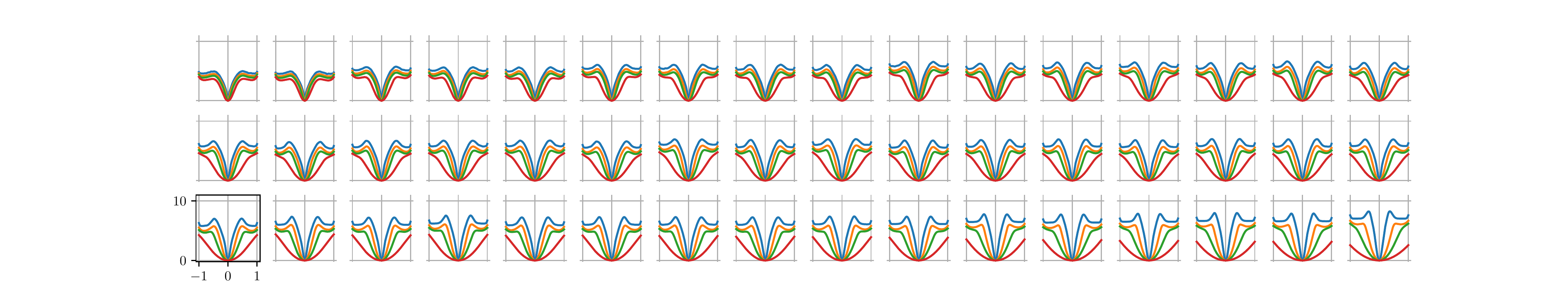}
	\vspace{-.7cm}
	\caption{%
		Learned filters \( k_j \) (top, the intervals show the values of black and white respectively, amplified by a factor of \num{10}) and potential functions \( -\log \psi_j \) (bottom) for a model trained only on \( \sigma = 0.02 \).
	}%
	\label{fig:patch results 5}
\end{figure*}
To emphasize the importance of the diffusion for learning, we learn a model solely on \( \sigma = 0.02 \) and compare the learned potential functions to the empirical marginal filter responses.
In detail,~\cref{fig:discrepancy} shows the normalized mean-squared error \( \mathrm{NMSE}_\kappa : \sqrt{2t} \mapsto \sum_{j=1}^J \fint_{\Omega_j(t, \kappa)} \frac{(\psi_j(z, \mathbf{w}_j, t) - h_j(z, t))^2}{\max_{z\in\Omega_j(t,\kappa)}h_j^2(z, t)}\ \mathrm{d}z \).
To avoid regions in which the empirical histogram is extremely unreliable, we define a \enquote{credible interval} \( \Omega_j(t, \kappa) = \{ [\alpha, \beta] \subset \R : \int_{-\infty}^\alpha h(\argm, t) = \int_\beta^\infty h(\argm, t) = \kappa \} \), and show different \( \kappa \in \{ \num{0.005}, \num{0.01}, \num{0.02} \} \).
The results show that training for multiple noise scales improves the performance.
In particular, larger diffusion times improve the performance also for (e.g.) \( \sigma = 0.02 \) \emph{especially in low-density regions}, which is apparent when \( \kappa \) is small.
When \( \kappa \) is large, we observe that the performance of the models only diverges when \( \sigma \) becomes relatively large.
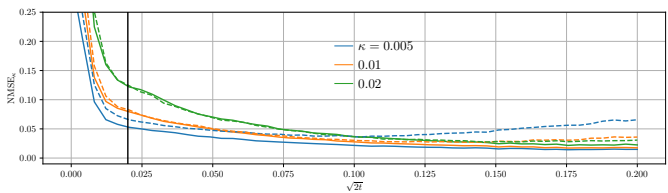
\begin{figure}
	\centering
	\includegraphics[width=\columnwidth]{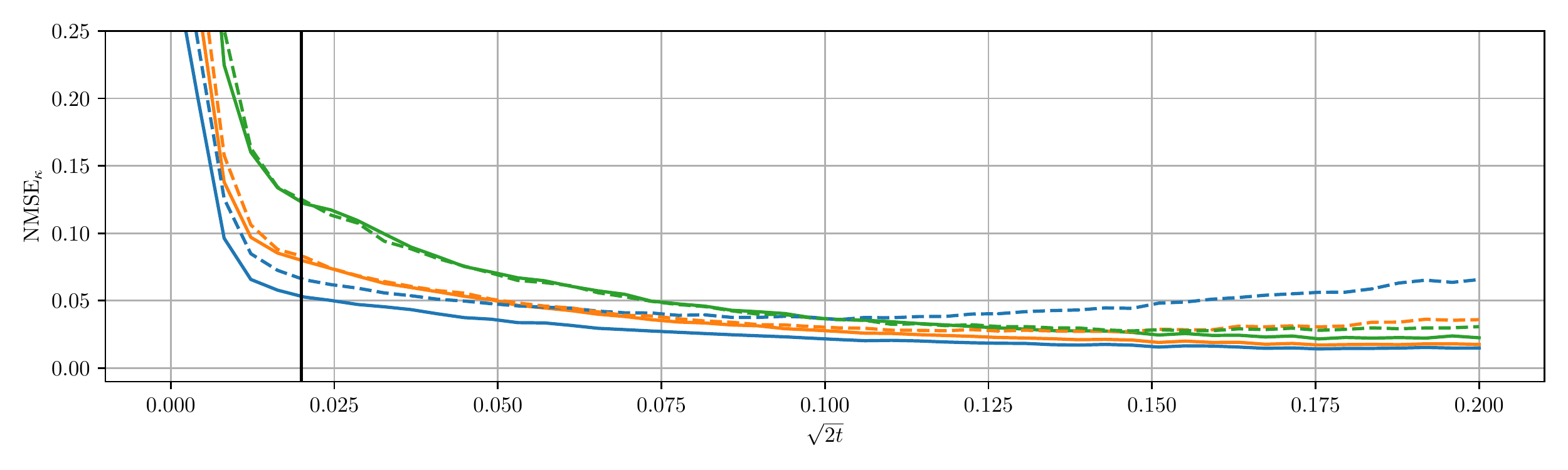}
	\scalebox{0.5}{
	\begin{tikzpicture}[overlay]
		\draw[mplblue, thick](0,5) -- ++(.5,0) node [right, black] {\(\kappa= \num{0.005}\)};
		\draw[mplorange, thick](0, 4.5) -- ++(.5,0) node [right, black] {\num{0.01}};
		\draw[mplgreen, thick](0, 4) -- ++(.5,0) node [right, black] {\num{0.02}};
	\end{tikzpicture}}
	\vspace{-.6cm}
	\caption{%
		Normalized mean-squared error between the learned potentials and the empirical marginals:
		For the solid lines, a model was trained on multiple noise scales, whereas only \( \sigma = \num{0.02} \) was used for the dashed lines.
	}%
	\label{fig:discrepancy}
\end{figure}

We show that our model is also scalable by learning on \( 15 \times 15 \) patches.
The subset containing \( 112 \) of the \( 15^2 - 1 \) filters and potential functions shown in \cref{fig:patch results 15} indicate that the results and discussion from the main paper also apply to much larger filter sizes.
The model shown in the figure was used in the denoising experiments in Table 1 in the main paper.
\begin{figure*}
	\centering
	\includegraphics[trim=4.5cm 1.5cm 4cm .6cm, clip, width=.86\textwidth]{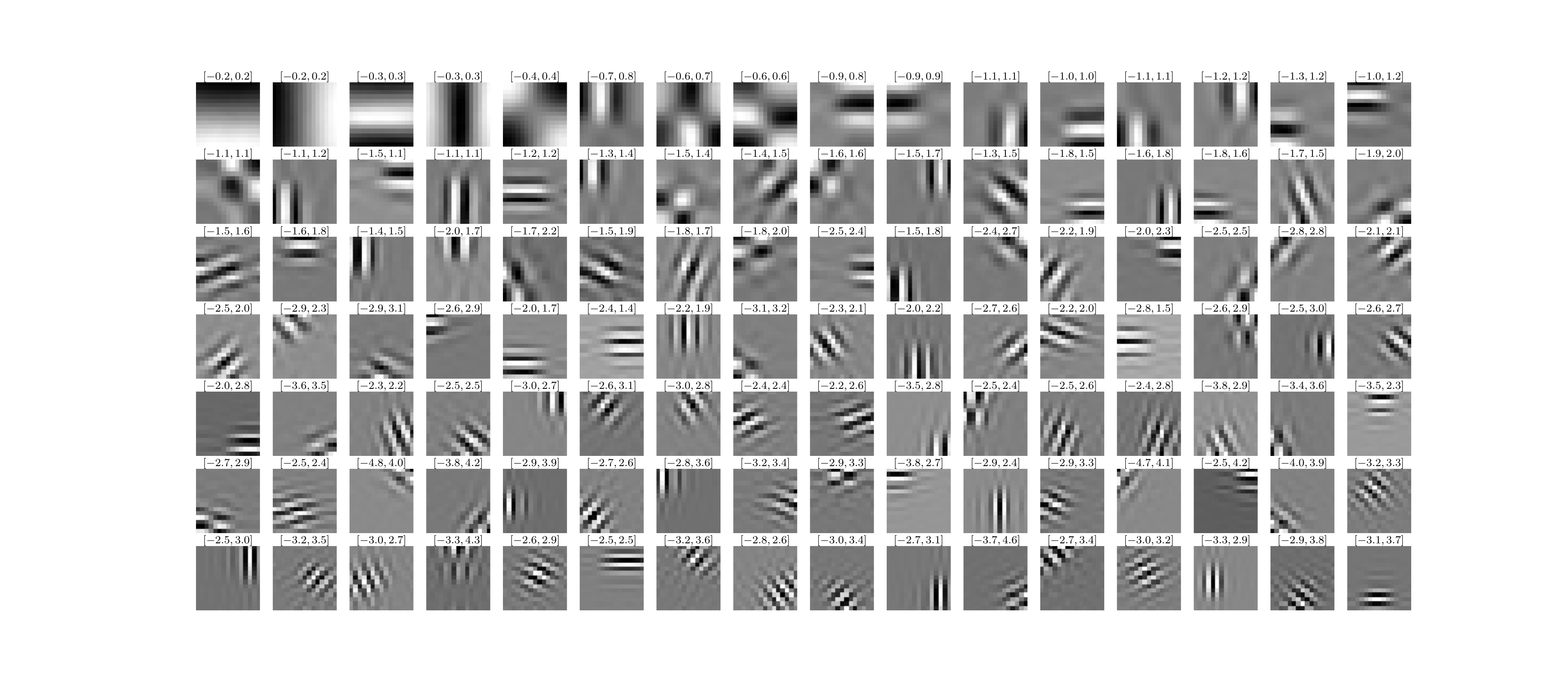}
	\includegraphics[trim=4.5cm 1.4cm 4cm 2cm, clip, width=.86\textwidth]{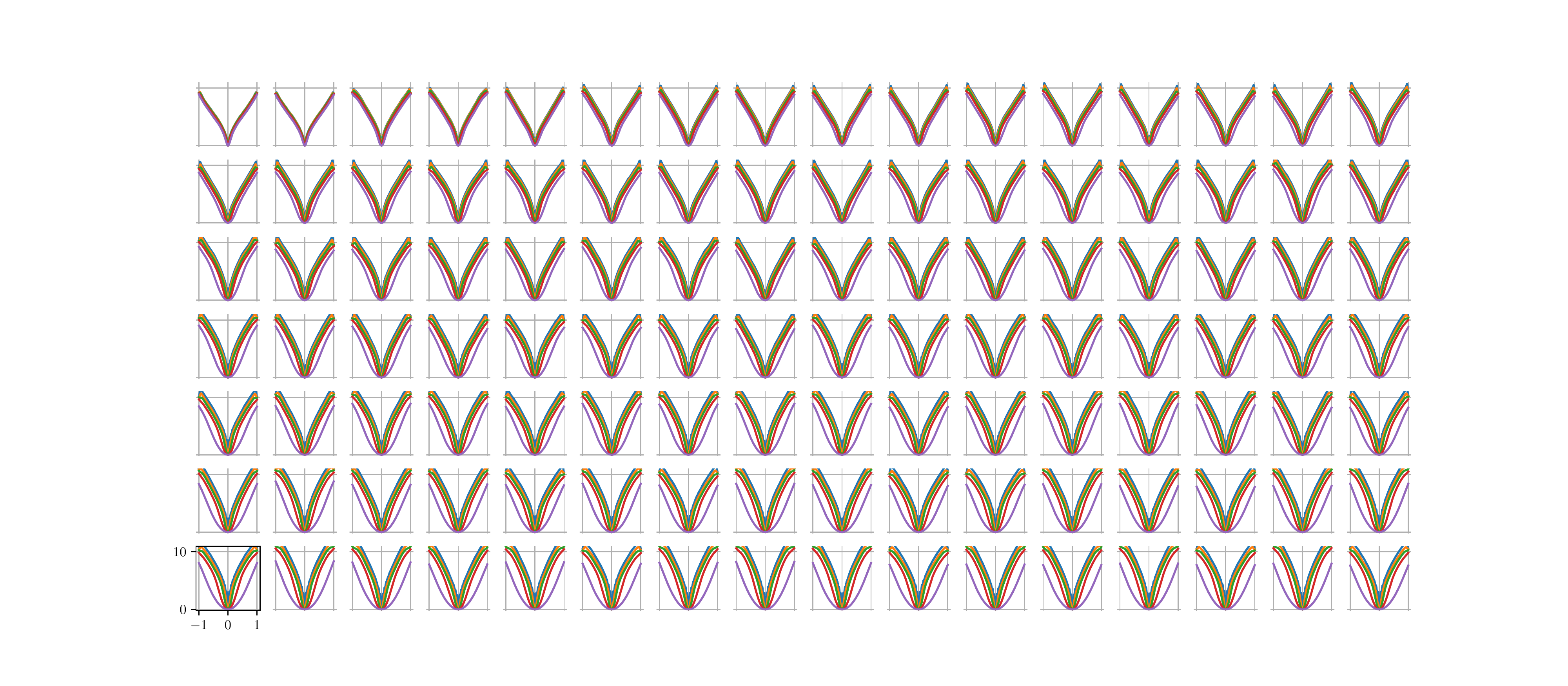}
	\rule[3mm]{.86\textwidth}{.3mm}
	\vspace*{-3mm}
	\includegraphics[trim=4.5cm 1.5cm 4cm 2cm, clip, width=.86\textwidth]{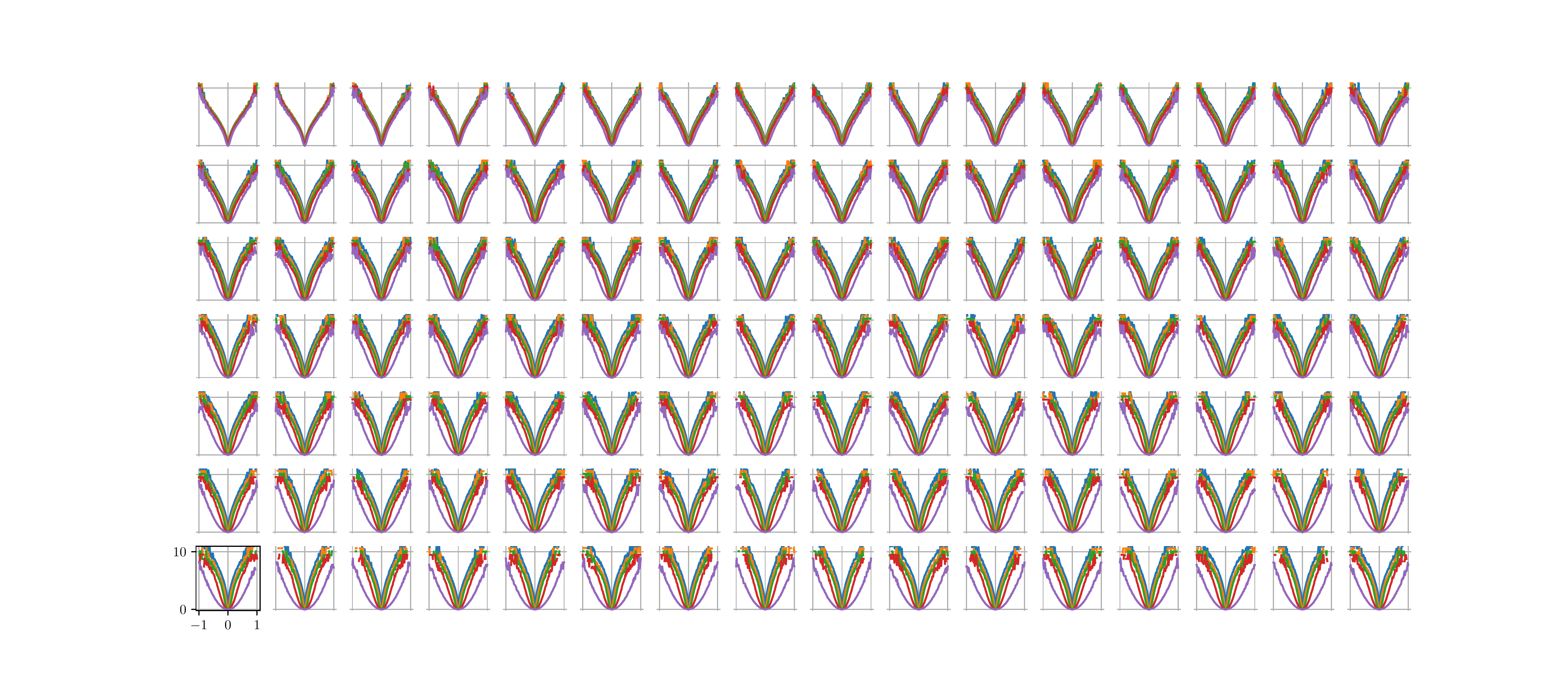}
	\caption{%
		Learned filters \( k_j \) (top, the intervals show the values of black and white respectively, amplified by a factor of \num{10}) and potential functions \( -\log \psi_j \) (bottom).
	}%
	\label{fig:patch results 15}
\end{figure*}
\bibliographystyle{splncs04}
\bibliography{bibliography}

\begin{thebibliography}{10}
\providecommand{\url}[1]{\texttt{#1}}
\providecommand{\urlprefix}{URL }
\providecommand{\doi}[1]{https://doi.org/#1}

\bibitem{bengio_representation_2013}
Bengio, Y., Courville, A., Vincent, P.: Representation learning: A review and
  new perspectives. IEEE Transactions on Pattern Analysis and Machine
  Intelligence  \textbf{35}(8),  1798--1828 (2013)

\bibitem{cole_green_2010}
Cole, K., Beck, J., Haji-Sheikh, A., Litkouhi, B.: Heat Conduction Using Greens
  Functions. {CRC} Press (Jul 2010)

\bibitem{efron_tweedie_2011}
Efron, B.: Tweedie’s formula and selection bias. Journal of the American
  Statistical Association  \textbf{106}(496),  1602--1614 (2011)

\bibitem{Gut2009}
Gut, A.: An Intermediate Course in Probability. Springer New York (2009)

\bibitem{hinton_training_2002}
Hinton, G.E.: Training products of experts by minimizing contrastive
  divergence. Neural Computation  \textbf{14}(8),  1771--1800 (8 2002)

\bibitem{hyvarinen_estimation_nodate}
Hyvarinen, A.: Estimation of non-normalized statistical models by score
  matching. Journal of Machine Learning Research p.~14 (2005)

\bibitem{Kirkpatrick1983}
Kirkpatrick, S., Gelatt, C.D., Vecchi, M.P.: Optimization by simulated
  annealing. Science  \textbf{220}(4598),  671--680 (May 1983)

\bibitem{martin_database_2001}
Martin, D., Fowlkes, C., Tal, D., Malik, J.: A database of human segmented
  natural images and its application to evaluating segmentation algorithms and
  measuring ecological statistics. In: Proc. ICCV. vol.~2, pp. 416--423 vol.2
  (2001)

\bibitem{miyasawa_empirical_1961}
Miyasawa, K.: An empirical bayes estimator of the mean of a normal population.
  In: Bulletin of the International Statistical Institute. pp. 161--188 (1961)

\bibitem{pock_inertial_2016}
Pock, T., Sabach, S.: Inertial proximal alternating linearized minimization
  ({iPALM}) for nonconvex and nonsmooth problems. {SIAM} Journal on Imaging
  Sciences  \textbf{9}(4),  1756--1787 (Jan 2016)

\bibitem{raphan_least_2011}
Raphan, M., Simoncelli, E.P.: {Least Squares Estimation Without Priors or
  Supervision}. Neural Computation  \textbf{23}(2),  374--420 (02 2011)

\bibitem{robbins_empirical_1956}
Robbins, H.: An empirical bayes approach to statistics. In: Proc. of the
  Berkeley Symposium on Mathematical Statistics and Probability. pp. 157--163
  (1656)

\bibitem{roberts_exponential_1996}
Roberts, G.O., Tweedie, R.L.: Exponential convergence of langevin distributions
  and their discrete approximations. Bernoulli  \textbf{2}(4),  341 -- 363
  (1996)

\bibitem{RoBl09}
Roth, S., Black, M.J.: {Fields of Experts}. Int. J. Comput. Vis.
  \textbf{82}(2),  205--229 (2009)

\bibitem{1591840}
Schrempf, O., Feiermann, O., Hanebeck, U.: Optimal mixture approximation of the
  product of mixtures. In: International Conference on Information Fusion.
  vol.~1, pp. 8 pp.-- (2005)

\bibitem{SoEr19}
Song, Y., Ermon, S.: Generative modeling by estimating gradients of the data
  distribution. In: Advances in Neural Information Processing Systems. vol.~32
  (2019)

\bibitem{song_scorebased_2021}
Song, Y., Sohl-Dickstein, J., Kingma, D.P., Kumar, A., Ermon, S., Poole, B.:
  Score-based generative modeling through stochastic differential equations.
  In: International Conference on Learning Representations (2021)

\bibitem{vincent_connection_2011}
Vincent, P.: A connection between score matching and denoising autoencoders.
  Neural Computation  \textbf{23}(7),  1661--1674 (2011)

\bibitem{zoran_learning_2011}
Zoran, D., Weiss, Y.: From learning models of natural image patches to whole
  image restoration. In: Proc. of the International Conference on Computer
  Vision. pp. 479--486. {IEEE} (2011)

\end{thebibliography}
\end{document}